 \documentclass[letterpaper]{article} 

\usepackage{wrapfig}
\usepackage{algorithm}
\usepackage{algorithmic}
\usepackage{graphics} 
\usepackage{epsfig} 
\usepackage{mathptmx} 
\usepackage{times} 
\usepackage{amsmath} 
\usepackage{amssymb} 
\usepackage{subfigure}
\usepackage{multirow}
\usepackage{amsthm}
\usepackage{authblk}

\newtheorem{theorem}{Theorem}

\newtheorem{assumption}[theorem]{Assumption}
\usepackage{color}

\newtheorem{definition}[theorem]{Definition}

\newtheorem{lemma}[theorem]{Lemma}

{
\newtheorem{remark}[theorem]{Remark}
\newtheorem{example}[theorem]{Example}
}

\usepackage{array}
\newcolumntype{P}[1]{>{\centering\arraybackslash}p{#1}}
\begin{document}
\title{Particle Clustering Machine: A Dynamical System Based Approach}
\author[1]{Sambarta Dasgupta\thanks{dasgupta.sambarta@gmail.com}}
\author[2]{Keivan Ebrahimi\thanks{keivan@iastate.edu}}
\author[2]{Umesh Vaidya\thanks{ugvaidya@iastate.edu}}
\affil[1]{Monsanto Company}
\affil[2]{ECPE Department, Iowa State University}
\maketitle
\begin{abstract} 
Identification of the clusters from an unlabeled data set is one of the most important problems in Unsupervised Machine Learning. The state of the art clustering algorithms are based on either the statistical properties or the geometric properties of the data set. In this work, we propose a novel method to cluster the data points using dynamical systems theory. After constructing a gradient dynamical system using interaction potential, we prove that the asymptotic dynamics of this system will determine the cluster centers, when the dynamical system is initialized at the data points. Most of the existing heuristic-based clustering techniques suffer from a disadvantage, namely the stochastic nature of the solution. Whereas, the proposed algorithm is deterministic, and the outcome would not change over multiple runs of the proposed algorithm with the same input data. Another advantage of the proposed method is that the number of clusters, which is difficult to determine in practice, does not have to be specified in advance. Simulation results with are presented, and comparisons are made with the existing methods. 
\end{abstract}
\section{Introduction}
Clustering is one of the most important and well studied problems in Unsupervised Machine Learning \cite{bishop2006pattern, murphy2012machine}, and Statistics \cite{arabie1996clustering}. Any clustering algorithm would take a set of unlabeled data points as inputs. The distance between the data points could be computed using some distance metric. The task of the clustering algorithm would be to group the data points in a manner such that the distances between data points within the same cluster is minimized, and the distances between data points belonging to different clusters are maximized.
There are several existing popular methods for clustering e.g. Hierarchal Clustering, Bayesian Clustering, K-means, C-means, Spectral Clustering, Mean-shift and so on. Some of these algorithms are iterative heuristics e.g. k-means, c-means. Mean Shift Clustering is based on feature space analysis to identify the minima of the density function \cite{fukunaga1975estimation, cheng1995mean}. Whereas, Spectral Clustering type algorithms \cite{ng2002spectral} try to analyze the algebraic properties of the distance matrix to cluster the data points. Application of cluster algorithms extends to various areas of unsupervised machine learning e.g. text mining natural language processing, computer vision, bio-informatics and so on \cite{wang2005introduction, manning2008introduction}. K-means algorithm is computationally less expensive, but suffers from the stochastic nature, as cluster centers changes in different runs. Bayesian Clustering utilizes Bayesian hierarchical structure to identify clusters, which is a model based method, and could potentially suffer if the model assumptions do not remain valid for a specific data set. Spectral Clustering involves eigen analysis of Laplacian matrix, and identifies number of data clusters from eigen gap. However, computationally Spectral Clustering would become cumbersome for large data sets. Also, Spectral Clustering uses k-means over the eigen representation of the data set after the dimensionality reduction, and thus also suffers from the stochastic nature of k-means.  \\
In this work, we propose a novel approach based on the Dynamical System Theory for solving the clustering problem. We construct a multi-agent gradient dynamical system with an interaction potential. We prove that the asymptotic dynamics of the multi-agent gradient dynamical system will help to identify the clusters, when the dynamical system is initialized at the location of data points. In particular, the cluster centers correspond to the fixed points of the gradient dynamical system. The key idea is that the data points are viewed as particles, and can exert forces on each other provided they are located sufficiently close to one another. This local interaction force field would steer the data particles to converge towards the appropriate cluster centers. The advantage of the proposed method is that the dynamics would allow to find the clusters without constructing a global potential or analyzing the algebraic properties of the metrics. The proposed approach for clustering draws some analogy with tools and methods from statistical physics \cite{chandler1987introduction}. In particular, the parameter determining of the neighborhood used in the construction of interaction potential is analogous to the temperature parameter in statistical physics. With an increase in system temperature, the entropy of the system increases leading to long range interaction among particles and this is analogous to the increasing size of the neighborhood in the interaction potential for multi-agent gradient system. Application of ideas from statistical physics to clustering is not new \cite{rose1998deterministic}. However, the approach proposed in this paper is different as the results are derived based on theory of dynamical systems. In \cite{dasgupta2017dynamical} a framework has been developed for supervised machine learning using dynamical systems theory. In this paper, we have developed a dynamical systems theory based famework for unsupervised problems. Our work also would be of interest to the researchers from Swarm Intelligence and Evolutionary Computing communities, who have been developing multi-agent heuristic distributed search, and optimization algorithms e.g. Genetic Algorithm (inspired from Darwinian natural selection) \cite{goldberg1988genetic}, Particle Swarm Optimizer (modeling social dynamics) \cite{eberhart1995new} and so on. \\
There are several advantages of the proposed method which are also the main contributions of this paper. The proposed algorithm is deterministic, and does not suffer from stochasticities involved in the meta-heuristics, for example K-means, or Spectral Clustering. In our proposed method, we do not require to carry out spectral analysis of the data sets (such as in Spectral Clustering), which would become cumbersome and computationally expensive for large data sets. The algorithm automatically determines number of clusters, which is formed by the number of disjoint domain of attraction basins of the particles state space. This is a major contribution considering the fact that determining number of clusters itself is a very challenging problem \cite{sugar2011finding,yan2007determining}. \\
The rest of the paper is structured as follows: Section \ref{sec_prob} describes the dynamical system based framework for clustering. The algorithm is described in \ref{sec_algo}, and subsequently we describe the convergence conditions in section \ref{sec_stability}. The method is extended to graph clustering in section \ref{sec_DPGC}. Simulation results are presented in \ref{sec_sim} and finally, conclusions are drawn in Section \ref{sec_con}.
\section{Dynamical system for clustering} \label{sec_prob}
Let us begin by considering a scenario where the unlabeled data set, which needs to be clustered, consists of $N$ data elements. An individual data point is denoted as $z_i \in \mathbb{R}^M, ~ i = 1, \dots N$, where $M$ is the dimension of the feature space. The objective is to construct a multi-agent dynamical system, which can recover the clusters of the data points, i.e., $\{z_1,\ldots,z_N\}$. In particular, the cluster centers for the data points can be identified from the steady state dynamics of the system. Let the multi-agent dynamical system be of the form,
\begin{eqnarray}
\dot x_k=f_k(x_1,\ldots,x_N),\;\;k=1,\ldots,N\label{dynamics}
\end{eqnarray}
with $x_k\in \mathbb{R}^M$. We will show that the steady state dynamics of this system (\ref{dynamics}) would converge to the centers of the appropriate clusters, when the system is initialized at the data points i.e., $x_i(0)=z_i$ for $i=1,\ldots,N$. The vector field $f_k$ will be constructed such that the stable fixed points of the dynamical system (\ref{dynamics}) will correspond to the center of the clusters where the points belonging to a particular cluster converge to the corresponding stable fixed points of the system. The objective is to design the dynamical system vector field i.e., $f_k(x)$ for $k=1,\ldots,N$. Towards this goal, we define a potential function as follows. 
\begin{definition}\label{def_pot}
The potential function, $U:\mathbb{R}^+\geq 0\to \mathbb{R}$ satisfies $U(r)\leq 0$, with minimum at the origin and $U(r)=0$ for $r\geq r_*$, and $\frac{d U}{dr}\geq 0$. With $r=\parallel x_i-x_j\parallel$, we further assume that the potential function $U$ has the property that, 
\[\frac{\partial U}{\partial x_i}=\varphi(r)\left(x_i-x_j\right)\]
where $\varphi$ again satisfies,
\begin{itemize}
\item $0 \leq \varphi(r)\leq 1$.
\item $\varphi(r)=0$ for $r\geq r_*$.
\item $1 > \varphi(r)> 0$,  ~ ~ $0 < r < r_*$.
\item $\frac{d \varphi}{d r} < 0 , ~~ 0 < r < r_*$. 
\end{itemize}
\end{definition}
Following are the examples of the potential function satisfying the above definition.
\begin{example}
\begin{eqnarray}
U(r)=\left\{\begin{array}{ccl}-\frac{(r^2-r_*^2)^2}{r_*^4}&{\rm if}& r\leq r_*\\
0&{\rm for}&r\geq r_*\\ 
\end{array}\right..
\end{eqnarray}
\end{example}
\begin{example}\label{example_gaussian}
\begin{eqnarray} \label{gaussian_potential}
U(r)=\left\{\begin{array}{ccl}-\exp^{-\left(\frac{r^2}{\sigma^2}\right)}&{\rm if}& r\leq r_*\\
0&{\rm for}&r\geq r_*\\ 
\end{array}\right..
\end{eqnarray}
\end{example}
\begin{remark}
The potential function, described in \eqref{gaussian_potential} changes its shape with modulation of $\sigma$. The potential function disappears after the limits of $3\sigma$; hence $r^*$ in this case can be seen as not other than $3\sigma$.
\end{remark}
The potential function in Definition \ref{def_pot} can be used to define interaction potential between particle $x_i$ and $x_j$ as follows
\begin{eqnarray}
\phi_{ij}(x_i,x_j)=U(\parallel x_i-x_j\parallel).
\end{eqnarray}
The total potential is then constructed using the individual interaction potential as follows:
\begin{eqnarray}
\Phi(x)=\sum_{i,j=1}^N\phi_{ij}(x_i,x_j).
\end{eqnarray}
The multi-agent dynamical system is then constructed using the  total potential function as following,
\begin{eqnarray}
\dot x_k&=&-\frac{\partial \Phi}{\partial x_k},\;\;k=1,\ldots, N \nonumber\\
&=&-\sum_{\ell=1}^N \frac{\partial \phi_{k\ell}}{\partial x_k}=\sum_{\ell=1}^N \varphi\left(\parallel x_k-x_\ell \parallel\right) (x_\ell -x_k)\label{gradient}.
\end{eqnarray}
Let $x=(x_1^\top,\ldots,x_N^\top)^\top$, then the vector dynamical system is described as
\begin{eqnarray} \label{vec_field}
\dot x=-\frac{\partial \Phi}{\partial x}.
\end{eqnarray}
Using the property of potential function in Definition \ref{def_pot}, the above dynamical system can be written as 
\begin{eqnarray} \label{sys_dyn}
\dot x=A(x)x.
\end{eqnarray}
The structure of $A(x)\in \mathbb{R}^{L\times L}$, with matrix $L=NM$ is as follows \[ A(x) :=- L(x) ~\otimes~ \hat {\mathbf{1}}, \]
where, $\otimes$ is the kronecker product, and $ \hat {\mathbf{1}}$ is the unit vector, and,
\begin{align*}
l_{ij}(x) &:= \varphi(\parallel x_i - x_j \parallel),~~~~~~~ i \neq j, \\
&:= -\sum_{k \neq i} \varphi (\parallel x_i - x_j \parallel), ~~i = j.
\end{align*}
The matrix $L(x)$ has Laplacian structure, as $L(x) \hat {\mathbf{1}} = 0 $ and $L(x) \ge 0$. The fixed points of the system $\dot x = - L(x) ~\otimes~ \hat 1 x$ would be the cluster centers. Also, it can be noted that $L\left( x \right )|_{x_i=z_i} $ is the Laplacian matrix, which is used in Spectral Clustering algorithm to identify the clusters. It is achieved by doing the eigen analysis of the Laplacian matrix with the state vector $x$ initialized at the initial data points i.e. $L\left( x \right )|_{x_i=z_i}$.\\
\begin{figure}
\centering
\includegraphics[width=0.55 \textwidth]{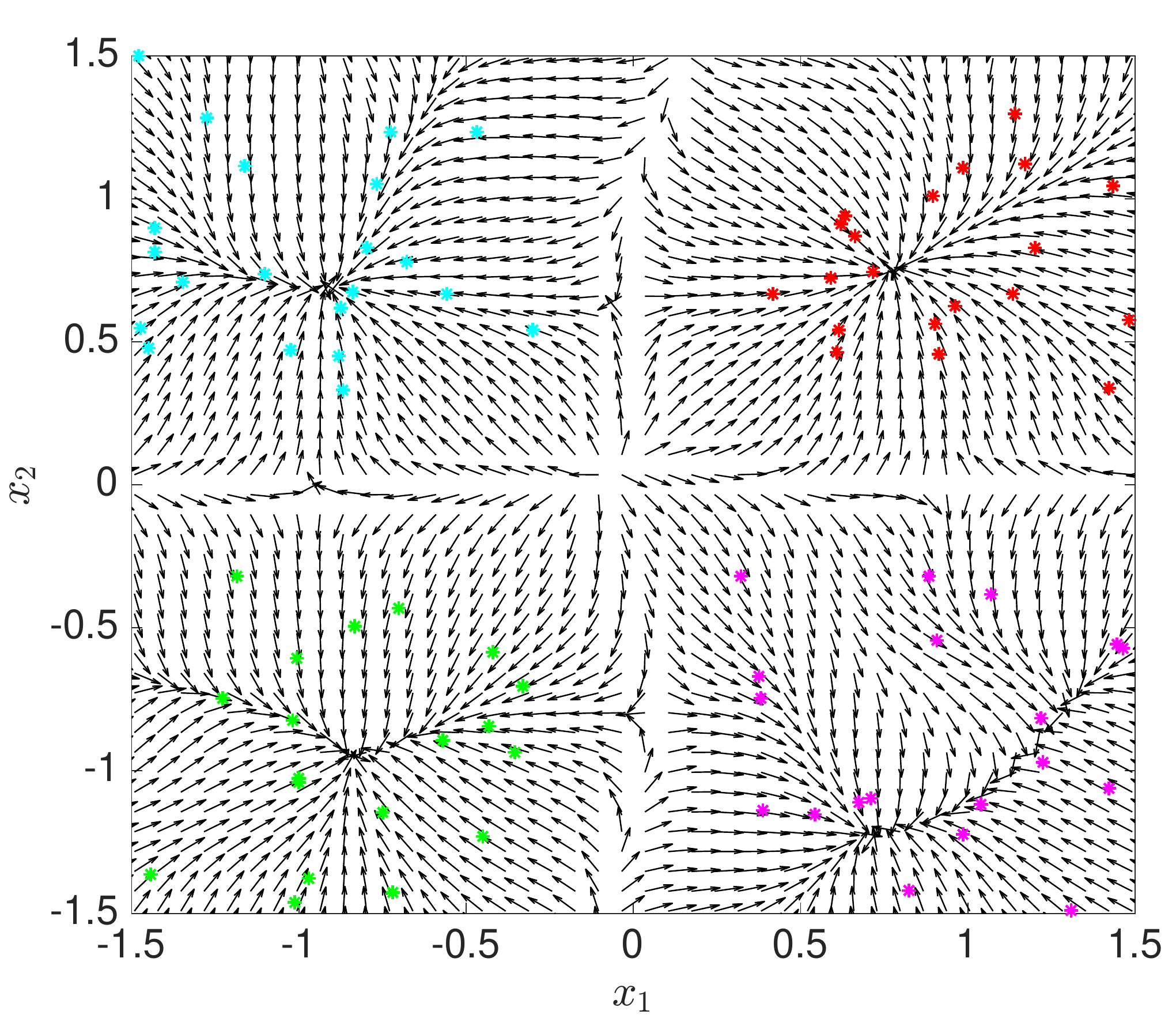}
\caption{ Set of initial data points alongside the velocity field, described by \eqref{vec_field} .}
\label{fig_vector_field}
\end{figure}
\begin{figure}
\centering
\includegraphics[width=0.55 \textwidth]{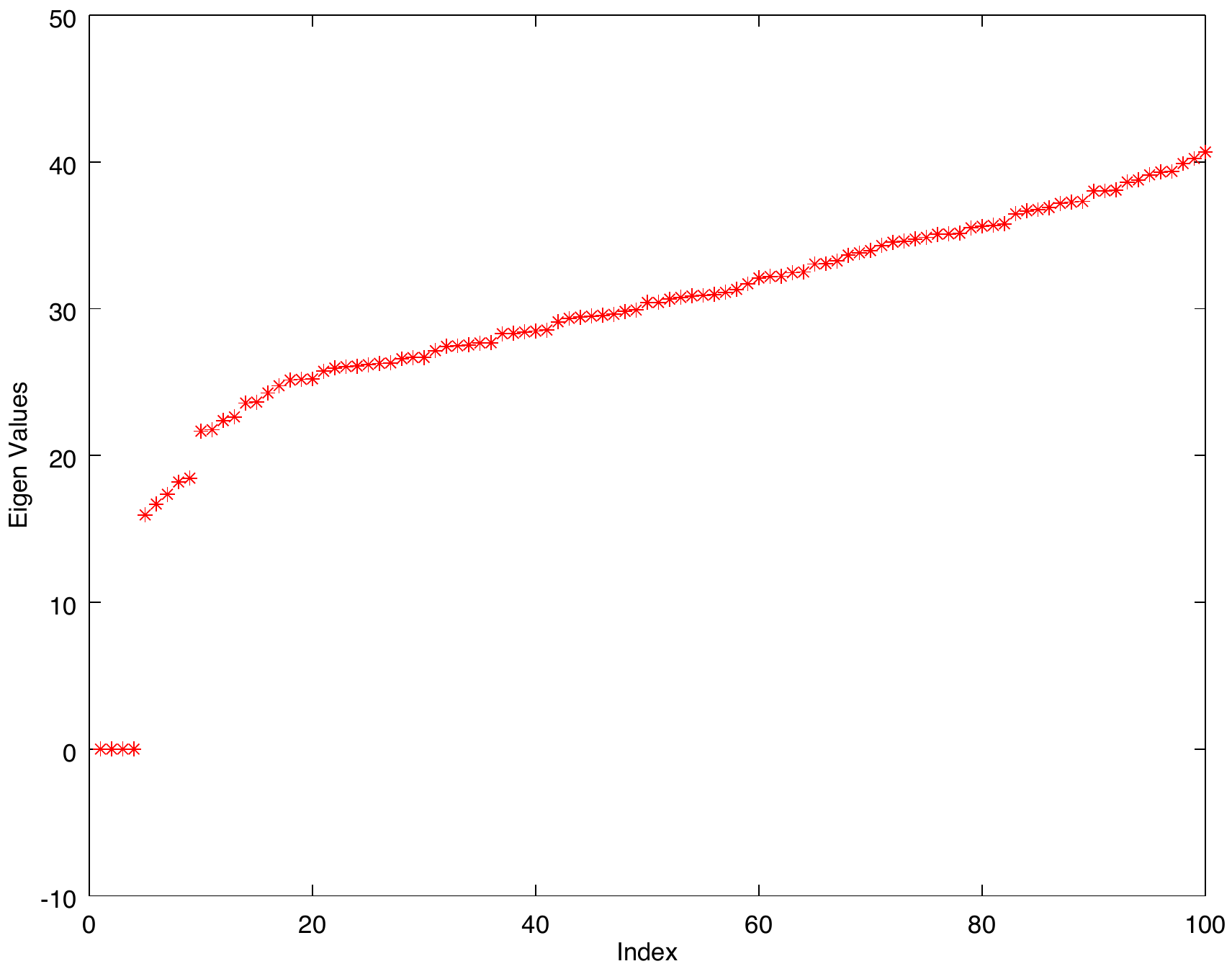}
\caption{ The eigen values of the Laplacian matrix $L\left ( x(0) \right ) $.}
\label{fig_eigen_values}
\end{figure}
\begin{figure}
\centering
\includegraphics[width=0.55 \textwidth]{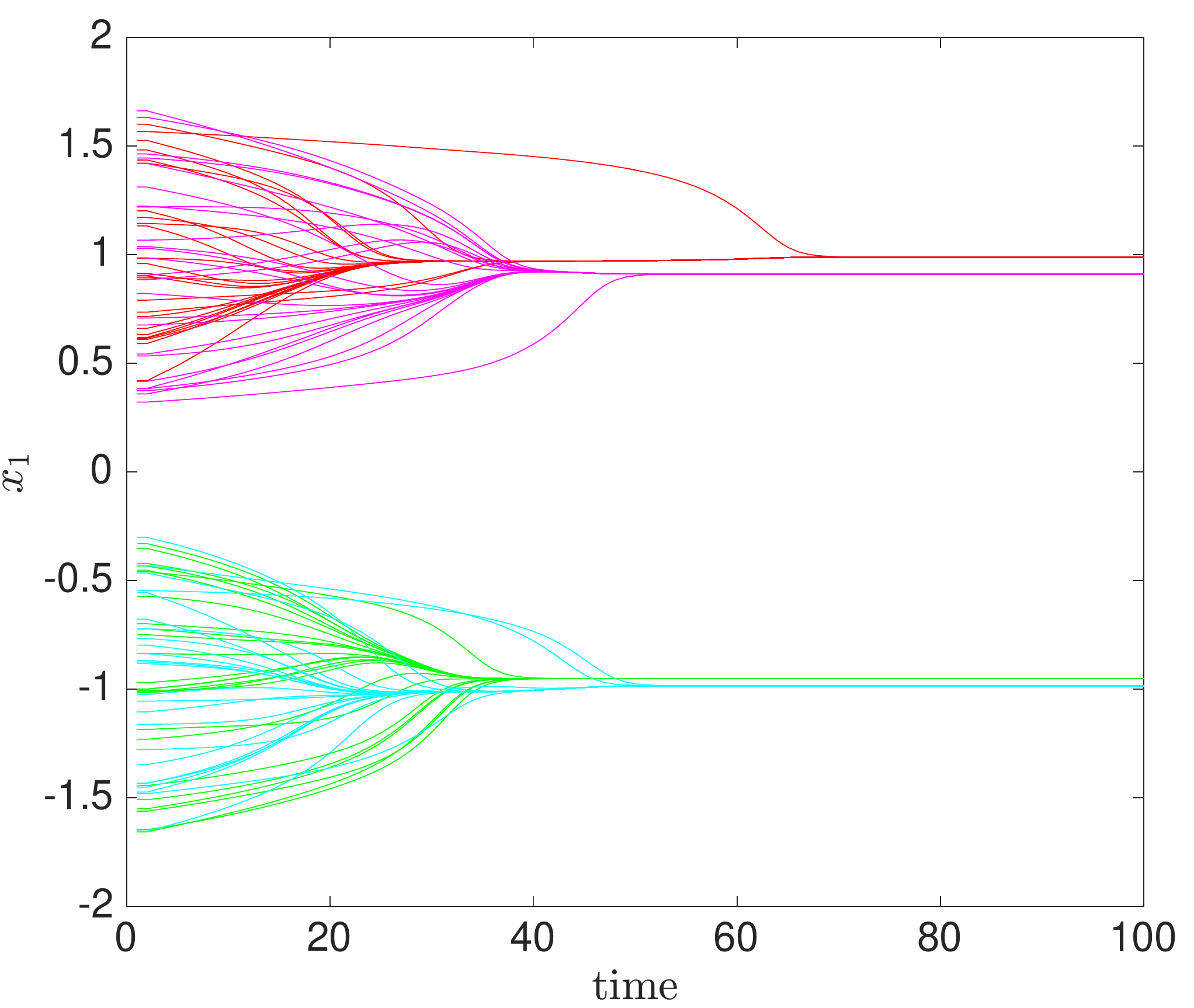}
\caption{ The evolution of $x_1$ dynamics of a $100$ particle system, comprised of $4$ clusters}
\label{fig_time_series_x1}
\end{figure}
Next, we carry out the eigen analysis of the Laplacian matrix for a synthetic data set, and also use the proposed method to cluster the same data set utilizing the dynamical evolution of the particles. It leads us to the conclusion that the both approaches in this example would lead to the identification of the same clusters. Fig. \ref{fig_vector_field} shows the initial position of data points and four clusters marked with four different colors. The Laplacian matrix $L(x)$ is constructed (taking exponential potential kernel as described in \eqref{gaussian_potential}) and the eigen values are plotted in Fig. \ref{fig_eigen_values}. From the theory of Spectral Clustering \cite{maimon2005data}, we know that the number of clusters can be identified as the indices of the eigen value where the eigen gap is maximized. As it can be obsereved from Fig. \ref{fig_eigen_values}, the eigen gap is maximized with four clusters. From the theory of spectral clustering, by clustering the dominant eigen modes using k-means, the clusters can be identified. In Spectral Clustering algorithm, these four dominant eigen modes would further be clustered, using K-means algorithm, to  identify the clusters \cite{maimon2005data}. The eigen decomposition of the Laplacian matrix would become a cumbersome task, if the number of data points are huge. Also, the K-means clustering would introduce stochastic nature to the outcome of the clustering process. Another relevant clustering algorithm, which is similar to the proposed method is the mean shift algorithm \cite{fukunaga1975estimation, cheng1995mean}. Mean shift algorithm is an iterative mode-seeking model. The algorithm tries to estimate the clusters based on kernel distribution functions (most commonly used kernel is Gaussian). It updates the means of the clusters in an iterative fashion using the gradient. The algorithm has similarities with the proposed approach as it tries to estimate the clusters iteratively based on a Gaussian kernel iteratively using the gradient. However, the MeanShift tries to estimate the centers directly as opposed to the proposed method, where data points evolve with an interaction dynamics. The proposed method does not attempt at estimating the cluster centers directly. \\
In our proposed method, we would leverage the dynamics of the multi-agent dynamical system (\ref{sys_dyn}) driven by the interaction potential to identify the cluster centers. We show that when the dynamical system is initialized at the data points, it will converge to the stable equilibrium points corresponding to the centers of clusters. In Fig. \ref{fig_time_series_x1}, the trajectory simulations of the multi-agent dynamical system are presented for the $x_1$ variable, and similar plots could be obtained for $x_2$ variable as well. The results are obtained using exponential potential function in Example \ref{example_gaussian} with $\sigma=0.15$. We notice that the asymptotic dynamics of the particles would eventually decide the four clusters without ambiguity. The stable equilibrium point of the dynamical system would provide us with the cluster centers. In the section \ref{sec_stability}, we will rigorously prove these results.
\section{Particle Clustering Algorithm}\label{sec_algo}
\begin{algorithm}[h!]  
\label{algo}
\caption{PCM Algorithm}
 $x_i(0) = z_i(0)$, $e_{ij} (0) = x_i (0) - x_j (0) $ ,  \\
$S(0) = max(\parallel e_{ij} (0) \parallel$ ) \\
 LOOP : \\
 for ~~$\tau$ ~in ~$0 \dots T$ \\
 for ~~$( i , j )$ ~ in ~ $1 \dots N$ \\
 $ x_{i} (\tau+1) = x_{i} (\tau) + \Delta t \sum_{k} \varphi_{ik}(\tau) \left ( x_{k} (\tau) - x_{i} (\tau)\right) $\\
 end \\
 $S(\tau+1) = \sum_{ij} | e_{ij} (\tau+1) - e_{ij} (\tau) | $ \\
 if ( $S(\tau+1) < N r_*$ ) break \\
 end \\
 Partition the data points into sets s.t. $ e_{ij} (\tau + 1) < r_* $, which would form the clusters. Merge clusters, which have sizes smaller than a threshold to the bigger ones based on the average distance. 
\end{algorithm}
The pseudo-code for PCM algo is presented in Algorithm \ref{algo}. In the algorithm we keep iterating the dynamics. The variable $S(\tau)$ is a dynamic measure of the dispersion. If it goes below a threshold the iteration stops. Finally, the distances are partitioned into groups to identify the clusters. Next, we would make some observation in terms of the computational complexity of the proposed method, and make comparisons with complexities of other algorithms. Finally, we propose a parameter tuning method to reduce the complexity of the algorithm.
\subsection{Complexity and Parameter Tuning} 
The computational complexity of Spectral Clustering is $\mathcal{O}(N^3)$, where $N$ is number of datapoints. Additionally, clustering of the dominant eigen modes would incur additional complexity $\mathcal{O}(r k m)$, where $r$ is the size of reduced dimensional space, $k$ is the number of iterations, and $m$ is the number of clusters . The complexity of K-means algorithm is $\mathcal{O}(N k m)$, where $N$ is the number of datapoints, $k$ is the number of iterations, and $m$ is the number of clusters. It is worth noting that both K-means and Spectral Clustering have stochastic heuristic in them, and as a consequence the performance of individual runs would vary. PCM on the other hand is a deterministic algorithm, and would produce same clusters always, whenever run on the same dataset. The complexity of the PCM algorithm is $\mathcal O(N^2 T)$, where $T$ is the number of iteration of PCM. The main computation of the PCM algorithm comes from updating the interaction potential $\phi_{ij}$, between datapoints $i$ and $j$. In order to reduce the complexity of the algorithm, we can remove those pairs from computation of interaction if $\parallel x_i - x_j \parallel < \epsilon$, we combine those data points, and accordingly the mass of the data point would increase, or in other words the interaction potential emanating from that point should be weighed by number of data point,which are combined. On the other hand, if  $\parallel x_i - x_j \parallel > \mathcal{T}$, where $\mathcal{T}$ is a large number, the pairs $(i,j)$ should be removed from the for loop in the PCM algorithm iterations. This way we can reduce the problem size, and also localize the for loops, which would make the computation sparse. With the aid of these two tricks, the computational complexity can be substantially reduced. \\ Another important aspect of the algorithm is to define the structure of the interaction potential, which is governed by $\sigma$ or $r^*$. It can come from the problem definition, and the prior knowledge about the distance metric can help to ascertain the parameter. In case of automatic tuning of the parameter, the following heuristic provides faster convergence on most cases, $ \sigma =  \frac{s.d.(\parallel e_{ij} \parallel)}{M}$, where $e_{ij} = x_i - x_j$, and $M$ is the dimension of the feature space. If one wishes to use a different potential function other than gaussian, the $r^*$ would be the critical parameter, which would determine the support of the interaction potential. Drawing parallels with the gaussian distribution, $r^*=3 \sigma$, as Gaussian function almost vanishes outside the $3 \sigma$ limits. This parameter tuning scheme is a general guideline, which could be further improved on case by case basis by further fine tuning. We have also chosen the parameter, $\epsilon$ same as that of the $\sigma$.
\section{Convergence Analysis} \label{sec_stability}
In this section, we prove the stability results for the multi-agent dynamical system (\ref{sys_dyn}). The stability results will allow us to connect the clustering problem with the asymptotic dynamics of the multi-agent system. In particular, we will prove that the particles, belonging to a particular cluster, would eventually converge to the mean of the initial positions of the data points, forming that cluster. The individual cluster centers (mean of data points belonging to the cluster) would form the equilibrium point of the multi-particle dynamical system. The proof is comprised of the following steps:
\begin{itemize}
\item First we would outline the assumption on the existence of multiple data clusters in the form of Assumption \ref {cluster_assumption}.
\item Next, in Theorem \ref{Theo1} we show that the dynamics of multi-particle system has the following sturcture of the equilirbrium point : the manifold where particles belonging to a specific cluster have same position, would form the equilibrium points. As a consequence, positions of particles within a cluster equal to mean position of data points belonging to the cluster forms the states of the equilibrium points of the system. 
\item Lemma \ref{single_cluster} proves convergence of the particles based on a Lyapunov function. The particles within a cluster synchronize and we construct the difference system i.e., distance between individual particles, and show that this system is stable by constructing a Lyapunov function. The stability of the difference system proves the convergence of the system.
\item Finally we prove equation \ref{lim_cluster} by using convergence of the particles in a cluster, and also using the fact that mean of the particles in a cluster is immobile.
\end{itemize}
Let the data set $z_1^N:=\{z_1,\ldots,z_N\}$ consists of $P$ clusters denoted by $S_{\ell}, ~ \ell =1,\dots P$. 
\begin{remark}\label{remark_pointlabel}With no loss of generality, we assume 
\[S_\ell=\{z_{N_{\ell-1}+1},\ldots, z_{N_{\ell}}\}.\]
Let $I_\ell$ be the index set corresponding to the data points in cluster $S_\ell$ i.e., 
\[I_\ell=\{N_{\ell-1}+1,\ldots,N_\ell\}.\]
Let the cardinality of the set $I_{\ell}$ be $n_{\ell}$. It is to be noted that, $n_1 = N_1$, and $n_{\ell} = N_{\ell}-N_{\ell-1}$ for $\ell > 1$.
\end{remark}
Essentially, the data points would only be segmented into clusters only when the inter cluster distance would dominate the intra cluster distances. Existence of clusters are stated in terms of the following assumptions on the inter cluster and intra cluster distances between the data points. 
\begin{assumption} \label{cluster_assumption} Let $i\in I_\ell$ be the index set corresponding to the cluster $S_\ell$ for $\ell=1,\ldots, P$. We assume that there exists an $r^*$ such that
\begin{itemize} \label{decoupled_eq}
\item $ \parallel z_{i} - z_{k} \parallel > r_* , ~~\forall ~ k \notin I_{\ell}\; \ \text{and} \;\;\forall\;\; i\in I_\ell$.
\item $ \parallel z_{i} - z_{j} \parallel< r_*, ~~ \forall ~ j \in I_{\ell}.$
\end{itemize}
\end{assumption}
Assumption \ref{cluster_assumption} states that the two data points belonging to two different clusters would always be separated by a minimum distance, where as two data points belonging to a same cluster would always be within a maximum distance of one another. We have following results characterizing the stability of multi-agent system in terms of the properties of the multiple clusters. 
\begin{theorem} \label{Theo1}
The equilibrium points of multi-agent dynamical system (\ref{sys_dyn}) are function of initial condition $x(0)$. Hence, for $x_k(0)=z_k$ for $k=1,\ldots, N$ and according to Remark \ref{remark_pointlabel} and Assumption \ref{cluster_assumption} on the data points, we have following characterization for one of the equilibrium point, $x^*=((x^*_1)^\top,\ldots,(x^*_N)^\top)$ of the system
\[x^*_k=\frac{1}{n_\ell}\sum_{j\in I_\ell} z_j,\;\;\;\;\forall\;\;k\in I_\ell.\]
\end{theorem} 
\begin{proof}
We have assumed $x_k(0)=z_k$. This would imply $\parallel x_j (0) - x_k (0) \parallel > r^*$ for all $j \notin I_{\ell}$, and $k \in I_{\ell}$. Let us consider the dynamics of the following sub-state,
\begin{eqnarray*}
\dot x_k & =\sum_{j=1}^N \varphi\left(\parallel x_k-x_j \parallel\right) (x_j -x_k), k \in I_{\ell}, \\
& =\sum_{j \in I_{\ell} } \varphi\left(\parallel x_k-x_j \parallel\right) (x_j -x_k),
\end{eqnarray*}
as for all $j \notin I_{\ell}$, and $k \in I_{\ell}$ accordinng to Assumption \ref{cluster_assumption}, $\parallel x_j - x_k \parallel > r^*$, and from property of the function $\varphi$ we get $ \varphi\left(\parallel x_k-x_j \parallel\right) = 0$. As a consequence, the dynamics of the particles within the cluster is decoupled from the rest of the particles. This can be observed that for the manifold $x_i = x_j, \forall ~ i,j \in I_{\ell}$, we would have $\dot x_k = 0$, and the manifold forms a continuum of equilibrium points. As a special case, the following values for the states $x_k = \frac{1}{n_\ell}\sum_{j \in I_{\ell}} z_{j}$ for all $ k \in I_{\ell}$, we get $\dot x_k = 0$. As the dynamics would be decoupled for each of the sets $I_{\ell}$, such point will be the equilibrium for each of the subsystems. Finally, we can construct the equilibrium of the entire system $x^*$ by stacking the individual equilibriums.
\end{proof}
The following theorem characterizes the local convergence of particles within clusters to these $P$ equilibrium points. We would construct a difference dynamics and show that the nonlinear dynamical system for particles belonging to a cluster would eventually converge i.e., the difference dynamics would be stable. The following theorem summarized the convergence of the particles:
\begin{theorem} \label{lim_cluster}
Let $x_{j}(t)$ be the state of the $j$ agent for the multi-agent dynamical system for $j \in I_{\ell} $. The data points $z_1^N$ satisfies Assumption \ref{cluster_assumption}. If $x_{j}(0)=z_{j}$, then
\[\lim_{t\to \infty} x_{j}(t)=x^\ell_*\]
for $j \in I_{\ell}$ and $\ell=1,\ldots,P$. 
\end{theorem}
\begin{proof}
As a consequence of $x_j(0) = z_j$, we would have $\parallel x_j (0) - x_k (0) \parallel > r^*$ for all $j \in I_{\ell}, k \notin I_{\ell}$. Under this condition and as have seen in the proof of the Theorem \ref{Theo1}, dynamics of the particles would be decoupled within the elements of each $I_{\ell}$. We would prove the result for one such set $I_{\ell}$, and then it can be extended to rest of the clusters. First, we would construct a Lyapunov function to show the stability of particles within a cluster, and then finally reach to the convergence result. Lyapunov function is widely used to prove stability of nonlinear systems \cite{khalil1996nonlinear,sastry2013nonlinear}. We take $N_{\ell}^{th}$ data point as reference and construct the following difference variable,
\[ d_{iN_{\ell}}(t) := e_{i}(t) - e_{N_{\ell}}(t) , ~~ i \in I_{\ell},\]
which leads us to the following equation
\begin{align*}
\dot e_{iN_{\ell}}(t) & := \dot x_{i}(t) - \dot x_{N_{\ell}}(t) , ~~ i \in I_{\ell} , \\
& := \sum_{k \in I_{\ell}} \varphi_{ki} e_{ki} -\sum_{k \in I_{\ell}} \varphi_{kN} e_{k N_{\ell}},
\end{align*}
where, $\varphi_{mn} := \varphi (\parallel x_m - x_n \parallel)$. We could construct the Lyapunov function as
\[ V_{\ell}(t) := \frac{1}{2} \sum_{i \in I_{\ell}}e^T_{i N_{\ell}} (t) e_{i N_{\ell}} (t).\] 
The time derivative of the Lyapunov function would be used to establish the convergence. First, we would assume there is only one cluster in the data set, and successively, we would extend it to cases where there are multiple clusters in the data points w.r.t. the Euclidean norm.
\begin{lemma} \label{single_cluster}
Lyapunov function for the system dynamics satisfies the following conditions, 
\begin{itemize}
\item $V_{\ell} \ge 0, \text{~ and ~} V_{\ell} = 0 \text{ iff } e_{iN_{\ell}} =0, ~\forall~ i \in I_{\ell} .$ 
\item If $ \parallel e_{ij} \parallel (0) \le r_* , ~ \forall i,j $, then $\dot{V}_{\ell}(t) < 0$.
\end{itemize}
\end{lemma} 
\begin{proof}
The first statement regarding the Lyapunov function is satisfied from the definition of the function. Next, we would try to prove the next part.
\begin{align*}
\dot V_{\ell}(t) & = \frac{1}{2} \sum_{i \in I_{\ell}} e^T_{i N_{\ell}} (t) \dot e_{i N_{\ell}} (t)\\
& =\frac{1}{2 ( n_{\ell} -1) } \sum_{i \in I_{\ell}} \left( \sum_{k \in I_{\ell}} (-e_{ki}+e_{k N_{\ell} })\right )^T \dot e_{i N_{\ell}}.
\end{align*}
Now, 
\begin{align*}
& \left( \sum_{k \in I_{\ell}} (-e_{ki}+e_{k N_{\ell} })\right )^T \dot e_{i N_{\ell}} = \\
& \left( \sum_{k \in I_{\ell}} (-e_{ki}+e_{k N_{\ell} })\right )^T \left( \sum_{k \in I_{\ell}} \varphi_{ki} e_{ki} -\sum_{k \in I_{\ell}} \varphi_{k N_{\ell}} e_{k N_{\ell}} \right ) = \\
& \sum_{k \in I_{\ell}} \left( - \varphi_{ki} \parallel e_{ki} \parallel^2 -\varphi_{k N_{\ell}} \parallel e_{kN_{\ell}}\parallel ^2 + (\varphi_{ki}+\varphi_{k N_{\ell}})e_{ki}^T e_{k N_{\ell}}\right) < \\
& \sum_{k \in I_{\ell}} \left( - \varphi_{ki} \parallel e_{ki} \parallel^2 -\varphi_{k N_{\ell}} \parallel e_{k N_{\ell}}\parallel ^2 + (\sqrt \varphi_{ki}+\sqrt \varphi_{k N_{\ell}})e_{ki}^T e_{k N_{\ell}}\right) \\
& = - \sum_{k \in I_{\ell}} \parallel (\sqrt \varphi_{ki} )e_{ki} + \sqrt \varphi_{k N_{\ell}} )e_{k N_{\ell}} \parallel^2 <0.
\end{align*}
Hence, $\dot V_{\ell}(t) < 0$.
\end{proof}
Once we have proved the stability of the particles belonging to cluster $S_{\ell}$, we would next prove the convergence. The existence of Lyaunov function proves convergence of the particles. We would also show that the particles would eventually converge to the mean position of data points which would form the cluster center. Let $\bar x_{\ell} (t) := \frac{1}{n_{\ell}} \sum_{ j \in I_{\ell}} x_j(t)$ and take time derivative as
\begin{align*}
\dot{ \bar x}_{\ell} (t) &= \frac{1}{n_{\ell}} \sum_{ j \in I_{\ell}} \dot x_j(t)\\
&= \frac{1}{n_{\ell}} \sum_{ j, k \in I_{\ell}} \varphi_{jk} (x_k - x_j)\\
&= \frac{1}{n_{\ell}} \sum_{i,j \in I_{\ell},~ j > k} \left ( \varphi_{jk} (x_k - x_j) + \varphi_{jk} (x_j - x_k)\right )\\
& = 0.
\end{align*}
The existence of Lyapunov function $V_{\ell}$ would imply $\underset{t \to \infty} {Lim}~~ e_{iN_{\ell}} (t) = 0$, which means $\underset{t \to \infty} {Lim}~~ x_{i} (t) = \underset{t \to \infty} {Lim}~~ \bar x_{\ell} (t)$. We have already proved $\dot{ \bar x}_{\ell} (t) =0 $. This means that $ \underset{t \to \infty} {Lim}~~ \bar x _{\ell} (t) = \bar x(0) $. This in turn implies, \[\lim_{t\to \infty} x_{j}(t)=x^\ell_* = \frac{1}{n_{\ell}}\sum_{k \in I_{\ell}} z_{k}, ~ j \in I_{\ell}. \] Hence, the proof.
\end{proof}
\begin{remark}
It is to be noted that the Theorem \ref{lim_cluster} provides a sufficiency condition for the convergence. The system might converge even if the condition is violated, as the condition is not necessary for the convergence.
\end{remark}
\section{Particle Graph Clustering} \label{sec_DPGC}
So far we have developed the algorithm, where the data points have real co-ordinates or in other words embedded in real Euclidean space. There are many  scenarios, where the clustering needs to be performed over datasets, where the data points are not embedded in a space with real coordinates. Graph clustering is a typical example, where distance between data points are provided, and the nodes of the graph is to be clustered, based on the distance. This framework can be extended to non-Euclidean metric spaces as well. \\
Let us consider first the clustering problem in Euclidean space as starter, and proceed further from there. We assume exponential interaction potential for this framework.
\begin{align} 
\label{pot_con}
 \phi(\parallel x_k - x_i \parallel ) := - e^{ - \frac{\parallel x_k - x_i \parallel^2 }{\sigma^2}}. 
 \end{align}
 The data points move according to the following equation,
\begin{align*}
\dot x_i(t) & = \sum_{k=1}^N  e^{ - \frac{\parallel x_k (t) - x_i (t) \parallel^2 }{\sigma^2}} \left ( x_k(t) - x_i (t) \right )
\end{align*}
 \begin{figure}
    \centering
    \includegraphics[width=0.55\textwidth]{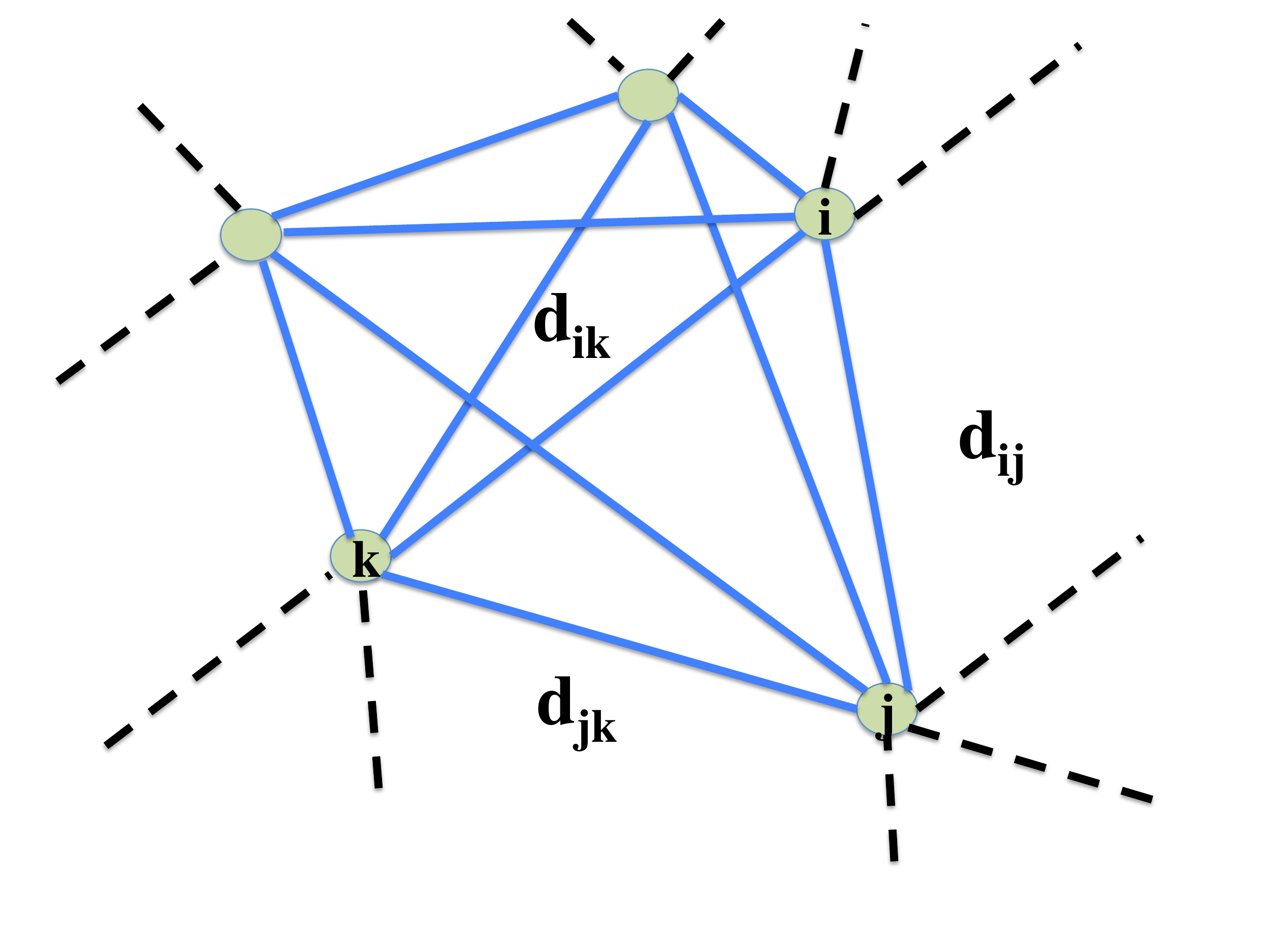}
  \caption{ Graph Clustering Schematic.}
 \label{fig_euclian}
\end{figure}
Next, the evolution of the distance between two data points is observed over time. The difference between the $i^{th}$ and $j^{th}$ ($i \neq j$)data points will be \[ d_{ij}^2(t) := \parallel x_i (t) - x_j (t) \parallel^2 = \left ( x_i(t) - x_j(t) \right )^T \left ( x_i(t) - x_j(t) \right ).\]
Differentiating w.r.t. time,
\begin{align*}
 \dot d_{ij}^2(t) &= 2 \left ( x_i(t) - x_j(t) \right )^T \left ( \dot x_i(t) - \dot x_j(t) \right ), \nonumber \\
 & = 2  \sum_{k=1}^N  e^{ - \frac{\parallel x_k (t) - x_i (t) \parallel^2 }{\sigma^2}}  \left ( x_i(t) - x_j(t) \right )^T  \left ( x_k(t) - x_i (t) \right ) \nonumber \\
 & + 2 \sum_{k=1}^N  e^{ - \frac{\parallel x_k (t) - x_j (t) \parallel^2 }{\sigma^2}}  \left ( x_i(t) - x_j(t) \right )^T  \left ( x_k(t) - x_j (t) \right ).
\end{align*}
Now, 
\begin{align*}
& \parallel x_j - x_k \parallel^2  =  \parallel x_j - x_i \parallel^2 +  \parallel x_i - x_k \parallel^2 + 2 \left ( x_i - x_j \right )^T  \left ( x_k - x_i \right ), \\
& \implies  \left ( x_i - x_j \right )^T  \left ( x_k - x_i \right )  = \frac{d_{jk}^2 - \left ( d_{ij}^2 + d_{ik}^2  \right )}{2}.
\end{align*}
Combining, 
\begin{align*}
& 2 d_{ij} \dot d_{ij}  =   \sum_{k=1}^N   [ e^{ - \frac{d_{ik} ^2 }{\sigma^2}}  \left (  d_{jk}^2 - \left ( d_{ij}^2 + d_{ik}^2  \right ) \right ) \\ &+ e^{ - \frac{d_{jk} ^2 }{\sigma^2}}  \left ( d_{ik}^2 - \left ( d_{ij}^2 + d_{jk}^2  \right ) \right )  ].\nonumber \\
 & \dot d_{ij}  =   \sum_{k=1}^N   [ e^{ - \frac{d_{ik} ^2 }{\sigma^2}}  d_{ik} \left ( \frac{ d_{jk}^2 - \left ( d_{ij}^2 + d_{ik}^2  \right ) }{2 d_{ij} d_{ik}  } \right ) \\ & + e^{ - \frac{d_{jk} ^2 }{\sigma^2}} d_{jk}  \left ( \frac{ d_{ik}^2 - \left ( d_{ij}^2 + d_{jk}^2  \right )}{2 d_{ij} d_{jk} } \right )  ]. \nonumber
\end{align*}
It can be noted, the evolution of the distance has become co-ordinate free. This way we can evolve the distance between data points, and let the edges of the graph evolve. After running the model over a time horizon, we classify the data points into same cluster if the final distance is below a small threshold, similar to what we have done in the PCM algorithm  before. The distance evolution also has a geometric interpretation. The distance evolution can be expressed as,
\begin{align}
\dot d_{ij}  =   \sum_{k=1}^N  \left [ e^{ - \frac{d_{ik} ^2 }{\sigma^2}}  d_{ik}  \cos(\pi - \delta_{jik}) + e^{ - \frac{d_{jk} ^2 }{\sigma^2}} d_{jk}  \cos(\pi - \delta_{ijk}) \right ] . \label{algo_eq}
\end{align}
where, $i \neq j$, 
\begin{align*}
\cos(\pi - \delta_{jik}) &= \frac{ d_{jk}^2 - \left ( d_{ij}^2 + d_{ik}^2  \right ) }{2 d_{ij} d_{ik}  }, \\
\cos(\pi - \delta_{ijk}) &=  \frac{ d_{ik}^2 - \left ( d_{ij}^2 + d_{jk}^2  \right )}{2 d_{ij} d_{jk} }
\end{align*}
$\delta_{jik}$ is the angle formed between line joining $j-i$ and $j-k$ in the triangle shown in Fig. \ref{fig_euclian}. It also can be noted that, $\dot d_{ii}  =0 $, which is also trivially satisfied by substituting $i=j$, and $\delta_{jik} = \delta_{ijk} = \frac{\pi}{2}$. The proposed differential equation of the distance between datapoints would evolve the graph, and make the nodes would converge to the respective clusters. This particular approach is coordinate-free, and the data points do not have to be embedded in any physical space. To run the dynamical system equation one only needs the initial distance between the datapoints.
\section{Simulation Results} \label{sec_sim}
In this section, we would simulate PCM over several data sets, and would draw comparisons with K-means, Spectral Clustering, and MeanShift algorithms.  However, it is to be noted that owing to the stochastic nature of K-means, and Spectral Clustering algorithms the outcomes produced by them might change across multiple runs, and the seeds of the random number generator. As a consequence, we run the K-means and Spectral Clustering $100$ times, and present the best case result for these two algorithms. The proposed algorithm PCM does not suffer from this shortcoming, and the outcomes would remain the same over multiple runs. We have used the statistical programming paradigm 'R' for running the experiments. We have used inbuilt K-means function in R, and for Spectral Clustering have used 'kernlab' R package \cite{kernlab}. For Mean Shift Clustering we have used 'MeanShift' R package \cite{r_meanshift}. For MeanShift the default kernel function and the automatic tuning scheme of the bandwidth (h - parameter) are adopted, which are inbuilt features of the R package. The comparisons are made among the algorithm with the aid of confusion matrices.  Once the confusion matrix is sorted to make it diagonal heavy, the off-diagonal entries capture the instances of erroneous classifications. We compute the sum of the off-diagonal entries under these conditions to estimate the total absolute error in the clustering. It is to be noted the extent of the errors would vary across multiple runs for k-means and spectral clustering. As a consequence we present the minimum, mean and standard deviation of the errors over multiple runs. However, for the PCM and MeanShift the error standard deviations are $0$. The computational times, which are presented cover $100$ iterations for K-Means, Spectral Clustering and for $1$ iteration for PCM and MeanShift, as the former two are stochastic and needed to be run multiple times. The number of clusters are determined automatically by PCM and MeanShift. Whereas, the number of clusters are to be specified for K-Means and Spectral Clustering, and we have provided the number from the prior knowledge. We also present the F-score as a measure of clustering accuracy, which is the harminc mean of precision and recall. All the codes, used to carry out the experiments, could be found in the following repository \cite{gitrepopcm}. 
\subsection{Clusters Generated by Multivariate Gaussian Mixture} 
\begin{figure}[h!]
\centering
\includegraphics[width=0.55 \textwidth]{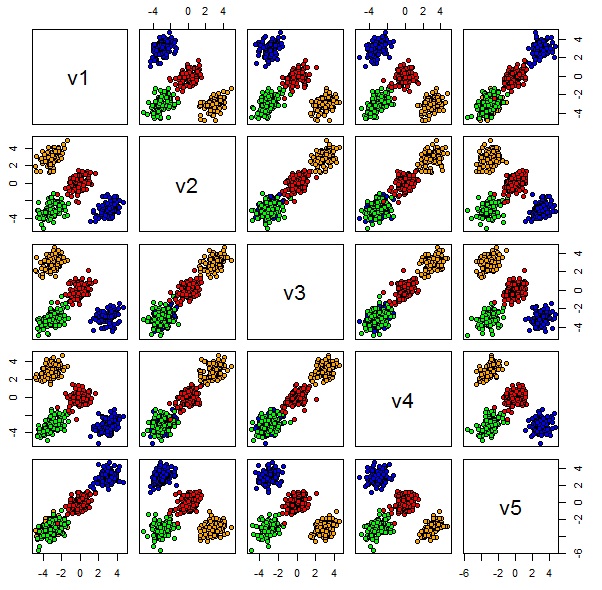}
\caption{ The paired scatter plot of Clusters Generated by Multivariate Gaussian Mixture.}
\label{fig_gaussian_vis}
\end{figure}
We consider a dataset, which is created by mixture of multivariate Gaussian distributions. The feature space has $5$ dimensions as can be observed from Fig. \ref{fig_gaussian_vis}.  There are $400$ samples generated, which are to be clustered. Different colors represent different clusters. The mixture  distribution is comprised of $4$ Gaussian distributions, centered at different locations. The covariance matrix for each of the Gaussian component is comprised of a positive definite matrix. We have used the R package 'MASS' to generate samples from the distributions. For Spectral Clustering and K-Means algorithms the number of cluster input was $4$. 
\begin{table}[h!] 
\caption{Confusion Matrices for Gaussian Mixture} \label{conf_gaussian}
\tiny
\begin{center}
\begin{tabular}[h!]
{|P{0.03cm}P{0.03cm}P{0.03cm}P{0.03cm}|P{0.03cm}P{0.03cm}P{0.03cm}P{0.03cm}|P{0.03cm}P{0.03cm}P{0.03cm}P{0.03cm}| P{0.03cm}P{0.03cm} P{0.03cm}|}
\cline{1-15} 
\multicolumn{4}{|c|}{K-Means}&\multicolumn{4}{|c|}{Spec Clus}&\multicolumn{4}{|c|}{PCM} &\multicolumn{3}{|c|}{MeanShift}\\
 \cline{1-15}
\multicolumn{1}{| P{0.03cm} } {100} & {0} & {0} & {0} & {100} & {0} & {0} & {0} & {100} & {0} & {0} & {0} & {100} & {0} & {0} \\ 
\cline{1-15}
\multicolumn{1}{| P{0.03cm}} {0} & {100} & {0} & {0} &  {0} & {100} &  {0} & {0} & {0} & {100} & {0} & {0} & {0} & {100} & {0} \\ 
\cline{1-15}
\multicolumn{1}{| P{0.03cm} } {0} & {0} & {100} & {0} &  {0} & {0} & {100} & {0} & {0} & {0} & {100} & {0} & {0} & {0} & {100} \\ 
\cline{1-15}
\multicolumn{1}{ |P{0.03cm} } {0} & {0} & {1} & {99} &  {0} & {0} & {1} & {99} & {0} & {0} & {1} & {99}  & {0} & {0} & {100}\\ 
\cline{1-15}
\end{tabular}
\end{center}
\end{table}
The confusion matrices for K-Means, Spectral Clustering, PCM, Meanshift are presented in the Table \ref{conf_gaussian}. It is to be noted that for K-Means and Spectral Clustering the best case confusion matrices are presented. The best case confusion matrices for K-Means and Spectral Clustering are comparable to that of the PCM. However, MeanShift identifies $3$ clusters as opposed to $4$, and combines two modes of the Gaussian into one.
\begin{table}[h!] 
\caption{Time, F-Score, error for Gaussian Mixture} \label{error_gaussian}
\scriptsize
\begin{center}
\begin{tabular}[h!]
{cc|c|c|c|}
\cline{2-5} 
\multicolumn{1}{ c|  }{\multirow{1}{*}{} } & \multicolumn{1}{|c|}{K-Means}  &  \multicolumn{1}{|c|}{Spec Clus} &  \multicolumn{1}{|c|}{PCM} 
 &  \multicolumn{1}{|c|}{Mean Shift}\\
  \cline{1-5}
\multicolumn{1}{ |c|  }{Time (sec) } & 
\multicolumn{1}{ |c| } {16.57} & {40.64} & {24.03} & {9.29} \\ 
 \cline{1-5}
  \multicolumn{1}{ |c|  }{Min Error } & 
\multicolumn{1}{ |c| } {1} & {1} & {1} & {100} \\ 
 \cline{1-5}
 \multicolumn{1}{ |c|  }{Mean Error } & 
\multicolumn{1}{ |c| } {71.6} & {44.79} & {1} & {100} \\ 
 \cline{1-5}
 \multicolumn{1}{ |c|  }{SD Error } & 
\multicolumn{1}{ |c| } {96.8} & {79.14} & {0} & {0} \\ 
 \cline{1-5}
  \multicolumn{1}{ |c|  }{Mean F-Score } & 
\multicolumn{1}{ |c| } {0.34} & {0.37} & {0.40} & {0.30} \\ 
 \cline{1-5}
 \multicolumn{1}{ |c|  }{SD F-Score } & 
\multicolumn{1}{ |c| } {0.09} & {0.07} & {0} & {0} \\ 
 \cline{1-5}
\end{tabular}
\end{center}
\end{table}
Table \ref{error_gaussian} contains the computational times for the four algorithms. The time intervals, which are presented for K-Means and Spectral Clustering, are the total time lapsed in $100$ iterations. Also, it contains the estimated errors in the clustering, which are computed by adding up the off-diagonal entries of the confusion matrices after sorting them to become diagonal heavy. As K-Means, and Spectral Clustering algorithms are stochastic in nature, we compute both the means and the standard deviations of the error in estimation. It can be observed K-Means has the least computational time. The minimum error (best case error) is same for K-Means, Spectral Clustering, and PCM. The mean error is least for PCM, which would remain constant due to the deterministic nature of PCM. As a consequence the standard deviation of error is $0$ for PCM. For MeanShift the Minimum and Mean errors values are the highest. The average F-Score values are highest and lowest for PCM and MeanShift respectively.
\subsection{Hyper-Spherical Clusters}
\begin{figure}[h!]
\centering
\includegraphics[width=0.55 \textwidth]{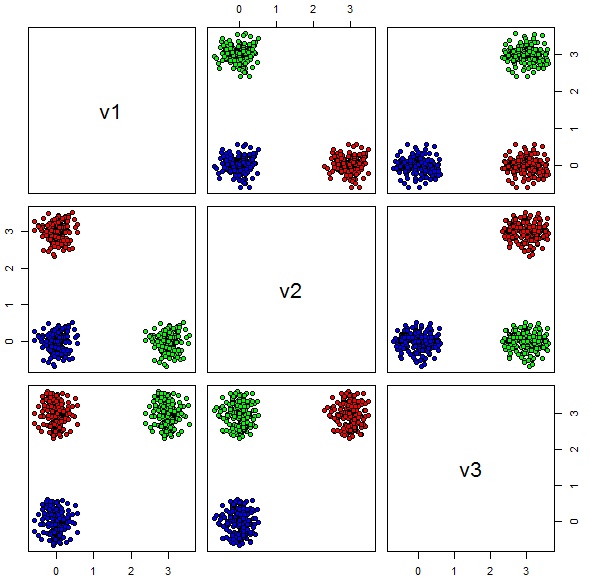}
\caption{The paired scatter plot of Hyper-Spherical clusters.}
\label{fig_shperical_vis}
\end{figure}
Next, we would consider the dataset, which is comprised of $600$ data points, which are distributed across $3$ hyper-spherical clusters in a $3$ dimensional feature space. Figure \ref{fig_shperical_vis} captures the paired scatter plots, where different colors represent different clusters. We have specified number of clusters as $3$ while running K-Means and Spectral Clustering.
\begin{table}[h!] 
\caption{Confusion Matrices for Hyper-Spherical Data set} \label{conf_spherical}
\scriptsize
\begin{center}
\tiny
\begin{tabular}[h!]
{P{0.1cm}P{0.1cm}P{0.1cm}|P{0.1cm}P{0.1cm}P{0.1cm}|P{0.1cm}P{0.1cm}P{0.1cm}| P{0.1cm}P{0.1cm}P{0.1cm}|}
\cline{1-12} 
 \multicolumn{3}{|c|}{K-Means}  &  \multicolumn{3}{|c|}{Spec Clus}
 &  \multicolumn{3}{|c|}{PCM} &  \multicolumn{3}{|c|}{MeanShift}\\
  \cline{1-12}
\multicolumn{1}{ |P{0.1cm} } {200} & {0} & {0} & {200} & {0} & {0} & {200} & {0} & {0}& {200} & {0} & {0} \\ 
\cline{1-12}
\multicolumn{1}{ |P{0.1cm} } {0} & {200} & {0} & {0} &  {200} & {0} &  {0} & {200} & {0} &  {0} & {200} & {0}  \\ 
\cline{1-12}
\multicolumn{1}{ |P{0.1cm} } {0} & {0} & {200} & {0} &  {0} & {200} & {0} & {0} & {200} & {0} & {0} & {200} \\ 
\cline{1-12}
\end{tabular}
\end{center}
\end{table}
Table \ref{conf_spherical} presents the confusion matrices for four algorithms. It needs to be noted that the confusion matrices for K-means and Spectral Clustering are the best case scenarios. The best case confusion matrices are same for K-Means and Spectral Clustering, which in turn is the same as that of PCM, and MeanShift.
\begin{table}[h!] 
\caption{Time, F-Score, error for Hyper-Spherical Data} \label{error_spherical}
\scriptsize
\begin{center}
\begin{tabular}[h!]
{cc|c|c|c|c|}
\cline{2-5} 
\multicolumn{1}{ c|  }{\multirow{1}{*}{} } & \multicolumn{1}{|c|}{K-Means}  &  \multicolumn{1}{|c|}{Spec Clus} &  \multicolumn{1}{|c|}{PCM} &  \multicolumn{1}{|c|}{MeanShift}\\
  \cline{1-5}
\multicolumn{1}{ |c|  }{Time (sec) } & 
\multicolumn{1}{ |c| } {0.84} & {63.98} & {13.99} & {8.11}\\ 
 \cline{1-5}
  \multicolumn{1}{ |c|  }{Min Error } & 
\multicolumn{1}{ |c| } {0} & {0} & {0} & {0} \\ 
 \cline{1-5}
 \multicolumn{1}{ |c|  }{Mean Error } & 
\multicolumn{1}{ |c| } {46.96} & {12.93} & {0} & {0} \\ 
 \cline{1-5}
 \multicolumn{1}{ |c|  }{SD Error } & 
\multicolumn{1}{ |c| } {108.23} & {58.87.59} & {0} & {0} \\ 
 \cline{1-5}
 \multicolumn{1}{ |c|  }{Mean F-Score } & 
\multicolumn{1}{ |c| } {0.46} & {0.47} & {0.5} & {0.5} \\ 
 \cline{1-5}
 \multicolumn{1}{ |c|  }{SD F-Score } & 
\multicolumn{1}{ |c| } {0.09} & {0.07} & {0} & {0}\\ 
 \cline{1-5} 
\end{tabular}
\end{center}
\end{table}
The computational time alongside the estimated errors are presented in the Table \ref{error_spherical}. It is to be noted that the time intervals of K-Means and Spectral Clustering are for $100$ iterations. It can be observed that the computational time is minimal for K-Means. The minimum, mean, and s.d. of error are $0$ in case of PCM and MeanShift. However, K-Means and Spectral Clustering render relatively high mean and s.d. of error due to the inherent stochastic components. The average value of F-Score slightly drops due to the inherent stochasticity in K-Means and Spectral Clustering.
\subsection{Concentric Hyper-Spherical Clusters}
\begin{figure}[h!]
\centering
\includegraphics[width=0.55 \textwidth]{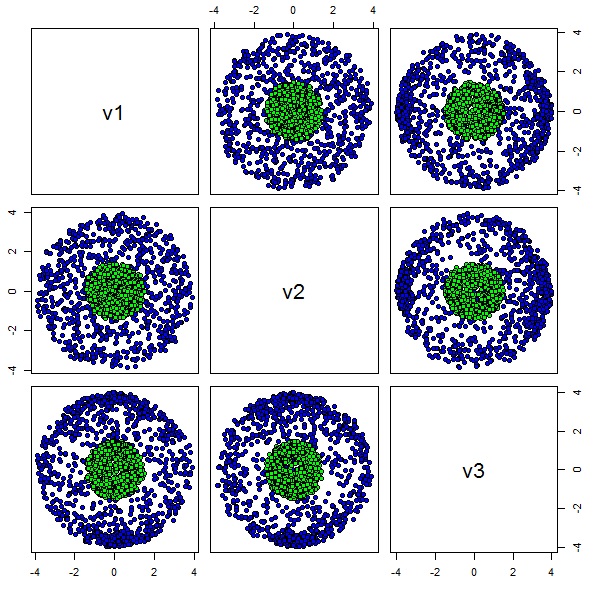}
\caption{The paired scatter plot of Concentric Hyper-Spherical clusters.}
\label{fig_concentric_vis}
\end{figure}
Figure \ref{fig_concentric_vis} shows set of $2000$ data points, which are distributed across two clusters, which are two concentric-hyper spheres in a $3$ dimensional feature space. The points marked by blue are embedded upon the outer sphere, and the green ones are on the inner sphere. 
\begin{table}[h!]
\caption{Confusion Matrices for Concentric Hyper-Spherical Data}  \label{table_conf_concentric}
\scriptsize
\begin{center}
\begin{tabular}[h!] 
{c|c|c|c|c|c|c|c|c|c|}
\cline{1-9} 
 \multicolumn{2}{|c|}{K-Means}  &  \multicolumn{2}{|c|}{Spec Clus}
 &  \multicolumn{2}{|c|}{PCM} &  \multicolumn{3}{|c|}{MeanShift}\\
  \cline{1-9}
\multicolumn{1}{ |c| } {530} & {470} & {1000} & {0} & {1000} & {0} & {0} & {1000} & {0} \\ 
\cline{1-9}
\multicolumn{1}{ |c| } {444} & {556} & {0} & {1000} &  {672} & {328} & {337} & {370} & {293} \\ 
\cline{1-9}
\end{tabular}
\end{center}
\end{table}
The number of clusters provided to K-Means and Spectral Clustering as input was $2$. Table \ref{table_conf_concentric} contains the confusion matrices for the four algorithms. It is to be kept in mind that the confusion matrices for K-Means and Spectral Clustering show the best case scenarios. The best case confusion matrix is superior for Spectral Clustering, which identifies the two clusters accurately. Spectral Clustering identifies the two clusters with complete accuracy in the best case scenario. The best case confusion matrix for K-Means is heavy along the off-diagonal. PCM identifies the inner hyper-sphere accurately as the first cluster. But PCM splits the outer hyper-sphere into two clusters. MeanShift automatically ascertains the number of clusters and finds $3$ clusters in this case. MeanShift keeps the inner hyper-sphere within one cluster, but splits the outer hyper-sphere into three clusters.
\begin{table}[h!] 
\caption{Time, F-Score, error for Concentric Hyper-Spherical Data} \label{error_concentric}
\scriptsize
\begin{center}
\begin{tabular}[h!]
{cc|c|c|c|c|}
\cline{2-5} 
\multicolumn{1}{ c|  }{\multirow{1}{*}{} } & \multicolumn{1}{|c|}{K-Means}  &  \multicolumn{1}{|c|}{Spec Clus}
 &  \multicolumn{1}{|c|}{PCM} &  \multicolumn{1}{|c|}{MeanShift}\\
  \cline{1-5}
\multicolumn{1}{ |c|  }{Time (sec) } & 
\multicolumn{1}{ |c| } {0.75} & {1344.22} & {163.59} & {325.48} \\ 
 \cline{1-5}
  \multicolumn{1}{ |c|  }{Min Error } & 
\multicolumn{1}{ |c| } {914} & {0} & {672} & {663} \\ 
 \cline{1-5}
 \multicolumn{1}{ |c|  }{Mean Error } & 
\multicolumn{1}{ |c| } {914} & {172.67} & {672} & {663} \\ 
 \cline{1-5}
 \multicolumn{1}{ |c|  }{SD Error } & 
\multicolumn{1}{ |c| } {0} & {271.60} & {0} & {0} \\ 
 \cline{1-5}
  \multicolumn{1}{ |c|  }{Mean F-Score } & 
\multicolumn{1}{ |c| } {0.362} & {0.61} & {0.45} & {0.41} \\ 
 \cline{1-5}
 \multicolumn{1}{ |c|  }{SD F-Score } & 
\multicolumn{1}{ |c| } {0} & {0.09} & {0} & {0}\\ 
 \cline{1-5}
\end{tabular}
\end{center}
\end{table}
Table \ref{error_concentric} captures the computational time, F-scores, and errors in estimates for the four algorithms. It is to be noted that the time intervals of K-Means and Spectral Clustering are for $100$ iterations. The minimum error or best case error is the minimum for Spectral Clustering. However, it takes relatively longer computational time for the computation in case of Spectral Clustering. Computational time is minimal for K-Means, but it renders maximum best case error as well. PCM produces higher value of the error in compared to the Spectral Clustering both in terms of mean and best case scenarios. The s.d. is $0$ for K-means, MeanShift, and PCM. Possibly, the K-means algorithm gets trapped in a local minima configuration in all runs for this particular dataset. MeanShift also takes considerable computation time, and produces slightly lesser F-scorein compared to PCM. It can be observed that for this particular dataset the Spectral Clustering clearly outperforms other algorithms in terms of the clustering accuracy. Spectral Clustering is based on the eigen function of the Laplacian matrix and can identify clusters even when the space is Non-Euclidean, which is the case for this particular dataset. However, PCM derives the dynamical equations assuming an Euclidean geometry of the feature space, and thus fails to capture the clusters accurately in this case. To make PCM suitable for Non-Euclidian geometry e.g. Spherical, Hyperbolic spaces, one needs to derive the graph clustering dynamics equations keeping in mind the geometry. The most generalized approach could be to assume the data points are embedded into a Riemannian Manifold, which admits a continuous Riemannian metric (tensor). One needs to use the Gaussian curvature and the tangent space dynamics to attain such a framework. This could be pursued in future research efforts to develop an extension of the PCM, where the data points are embedded into a smooth Riemmanian Manifold. 
\subsection{Edgar Anderson's Iris Data set}
\begin{figure}[h!]
\centering
\includegraphics[width=0.55 \textwidth]{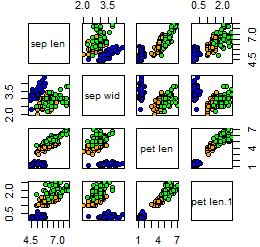}
\caption{ The pair scatter plot of Edgar Anderson's Iris Data.}
\label{fig_IrisData_Pair_Scatter_Plot}
\end{figure}
In this section we would present simulation studies over Edgar Anderson's Iris Data, which has been used extensively as a benchmark clustering data set in the past \cite{becker1988new}. We have used 'iris', which is an inbuilt realization of the data set in R. The data set contains $150$ samples of iris flowers. Each data point has $4$ variables namely sepal length, sepal width, petal length, and petal width. There are three species of iris flowers namely setosa, versicolor, and virginica, which we would use as the a priory labels while computing the confusion matrices. The iris data set with actual species names is depicted in the pair scatter plot of Fig. \ref{fig_IrisData_Pair_Scatter_Plot}, where each species is marked with different colors. \\
\begin{table}[h!] 
\caption{Confusion Matrices for Iris Data set}
\scriptsize
\label{conf_iris}
\begin{center}
\begin{tabular}[h!]
{P{0.1cm}P{0.1cm}P{0.1cm}|P{0.1cm}P{0.1cm}P{0.1cm}|P{0.1cm}P{0.1cm}P{0.1cm}|P{0.1cm}P{0.1cm}|}
\cline{1-11} 
 \multicolumn{3}{|c|}{K-Means}  &  \multicolumn{3}{|c|}{Spec Clus} & \multicolumn{3}{|c|}{PCM} &  \multicolumn{2}{|c|}{MeanShift}\\
  \cline{1-11}
\multicolumn{1}{ |P{0.1cm} } {36} & {0} & {14} & {49} & {0} & {1} & {50} & {0} & {0} & {0} & {50} \\ 
\cline{1-11}
\multicolumn{1}{ |P{0.1cm} } {0} & {50} & {0} & {0} &  {50} & {0} &  {0} & {50} & {0} & {50} & {0} \\ 
\cline{1-11}
\multicolumn{1}{ |P{0.1cm} } {2} & {0} & {48} & {3} &  {0} & {47} & {0} & {15} & {35} & {0} & {50}  \\ 
\cline{1-11}
\end{tabular}
\end{center}
\end{table}
We have provided $3$ as input number of clusters to K-Means and Spectral Clustering. Table \ref{conf_iris} shows the confusion matrices for four algorithms. It is to be noted that for the K-means and Spectral Clustering, the confusion matrices correspond to the best case scenarios. The best case scenario for K-Means and Spectral clustering are comparable in terms of misclassifications. PCM renders higher error w.r.t. these two algorithm's best case scenario. PCM distributes the Versilor instances into two clusters. MeanShift determines automatically the number of clusters to be $2$, and combines Virginia and Versicolor into one cluster. 
\begin{table}[h!] 
\caption{Time, F-Score, error for Iris Data}
\scriptsize
\label{error_iris}
\begin{center}
\begin{tabular}[h!]
{cc|c|c|c|c|}
\cline{2-5} 
\multicolumn{1}{ c|  }{\multirow{1}{*}{} } & \multicolumn{1}{|c|}{K-Means}  &  \multicolumn{1}{|c|}{Spec Clus} 
 &  \multicolumn{1}{|c|}{PCM} &  \multicolumn{1}{|c|}{MeanShift}\\
  \cline{1-5}
\multicolumn{1}{ |c|  }{Time (sec) } & 
\multicolumn{1}{ |c| } {0.67} & {16.55} & {5.67} & {1.83}\\ 
 \cline{1-5}
  \multicolumn{1}{ |c|  }{Min Error } & 
\multicolumn{1}{ |c| } {16} & {4} & {15} & {50} \\ 
 \cline{1-5}
 \multicolumn{1}{ |c|  }{Mean Error } & 
\multicolumn{1}{ |c| } {25.87} & {24.81} & {15} & {50} \\ 
 \cline{1-5}
 \multicolumn{1}{ |c|  }{SD Error } & 
\multicolumn{1}{ |c| } {22.22} & {21.09} & {0} & {0} \\ 
 \cline{1-5}
  \multicolumn{1}{ |c|  }{Mean F-Score } & 
\multicolumn{1}{ |c| } {0.41} & {0.42} & {0.45} & {0.33} \\ 
 \cline{1-5}
 \multicolumn{1}{ |c|  }{SD F-Score } & 
\multicolumn{1}{ |c| } {0.07} & {0.07} & {0} & {0} \\ 
 \cline{1-5}
\end{tabular}
\end{center}
\end{table}
Table \ref{error_iris} shows that the computational time is lowest for K-means and maximum for Spectral Clustering. The time intervals of K-Means and Spectral Clustering are for $100$ iterations. The best case error or minimum error in cluster estimation is minimum for Spectral Clustering and maximum for MeanShift. However, the mean and standard deviation are lowest for MeanShift, and PCM. The mean F-Score value is highest for PCM and lowest for MeanShift.
\subsection{Pima Indian Diabetes Data Set}
\begin{figure}[h!]
\centering
\includegraphics[width=0.55 \textwidth]{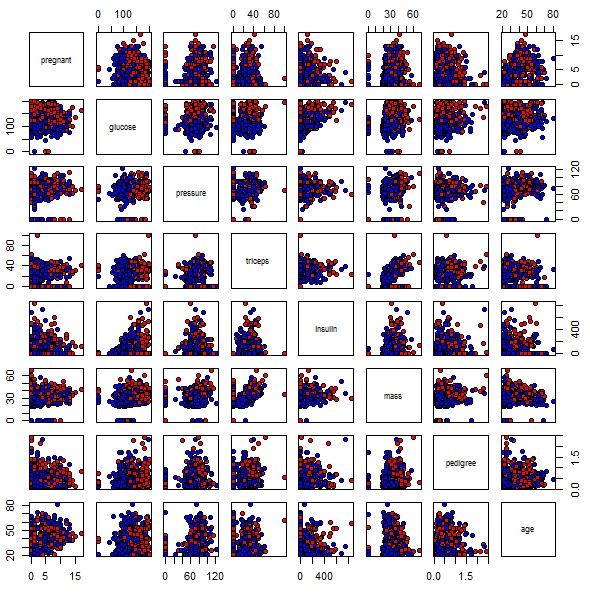}
\caption{Pima Indian Data Set.}
\label{fig_pimaindian}
\end{figure}
We would consider simulation studies on the Pima Indian Diabetes Data set. The data set is comprised of medical test results for Pima Indian patients, who were at least $21$ years of age \cite{smith1988using}. We have curated the data from the R-package 'mlbench' \cite{mlbench}. There are $768$ instances in the data set. The set of features is comprised of $8$ real measurements. Figure \ref{fig_pimaindian} shows the pairwise scatter plot of the Pima Indian Diabetes data set with positive and negative samples marked as red and green respectively.
\begin{table}[h!] 
\caption{Confusion Matrices for Pima Indian Data}
\tiny
\label{conf_pima_indian}
\begin{center}
\begin{tabular}[h!]
{|P{0.1cm} P{0.1cm}|P{0.1cm} P{0.1cm}|P{0.1cm} P{0.1cm}|P{0.1cm} P{0.1cm}P{0.2cm} P{0.1cm}P{0.1cm} P{0.1cm} P{0.1cm}P{0.1cm} P{0.1cm}|}
\cline{1-15} 
 \multicolumn{2}{|c|}{K-Means}&\multicolumn{2}{|c|}{Spec Clus}&\multicolumn{2}{|c|}{PCM}&\multicolumn{9}{|c|}{MeanShift}\\
\cline{1-15}
\multicolumn{1}{ |P{0.1cm} } {421} & {79} & {273} & {227} & {253} & {247} & {301       } &  {169} &  {4} &    {0}   &       {21}    &      { 3} & {1} & {1} & {0} \\ 
\cline{1-15}
\multicolumn{1}{ |P{0.1cm} } {182} & {86} & {114} & {154} &  {141} & {127} &  {141} &  {93} &   {11} &  {1} &    {19} &     {2} & {0} & {0} & {1} \\ 
\cline{1-15}
\end{tabular}
\end{center}
\end{table}
Table \ref{conf_pima_indian} shows the confusion matrices for four algorithms. For K-means and Spectral Clustering the confusion matrices are for the best case scenarios. The best case scenario confusion matrix points to highest accuracy for K-Means. The number of clusters, which is provided as an input to K-Means and Spectral Clustering, has been made equal to $2$ for these two algorithms. The MeanShift algorithm determines the number of clusters automatically, and in this particular case has identified $9$ clusters. As a consequence, the number of columns in the confusion matrix is $9$ for MeanShift.
\begin{table}[h!] 
\caption{Time, F-Score, error for Pima Indian Diabetes Data}
\scriptsize
\label{error_pima_indian}
\begin{center}
\begin{tabular}[h!]
{cc|c|c|c|}
\cline{2-5} 
\multicolumn{1}{ c|  }{\multirow{1}{*}{} } & \multicolumn{1}{|c|}{K-Means}  &  \multicolumn{1}{|c|}{Spec Clus} &  \multicolumn{1}{|c|}{PCM}  &  \multicolumn{1}{|c|}{MeanShift}\\
  \cline{1-5}
\multicolumn{1}{ |c|  }{Time (sec) } & 
\multicolumn{1}{ |c| } {0.77} & {229.13} & {25.03} & {29.92}\\ 
 \cline{1-5}
  \multicolumn{1}{ |c|  }{Min Error } & 
\multicolumn{1}{ |c| } {261} & {341} & {388} & {374}\\ 
 \cline{1-5}
 \multicolumn{1}{ |c|  }{Mean Error } & 
\multicolumn{1}{ |c| } {401.22} & {381.74} & {388} & {374} \\ 
 \cline{1-5}
 \multicolumn{1}{ |c|  }{SD Error } & 
\multicolumn{1}{ |c| } {122.40} & {11.19} & {0} & {0} \\ 
 \cline{1-5}
  \multicolumn{1}{ |c|  }{Mean F-Score } & 
\multicolumn{1}{ |c| } {0.34} & {0.33} & {0.32} & {0.33} \\ 
 \cline{1-5}
 \multicolumn{1}{ |c|  }{SD F-Score } & 
\multicolumn{1}{ |c| } {0.08} & {0.01} & {0} & {0} \\ 
 \cline{1-5}
\end{tabular}
\end{center}
\end{table}
Table\ref{error_pima_indian} provides the computational time, F-Score, error values for the four algorithms. It is to be noted that the time intervals for K-Means and Spectral Clustering show the total elapsed time for $100$ iterations. The best case error and mean error are lowest for K-means, and highest for PCM. For Spectral Clustering the computational time maximum. For PCM the error s.d. is least, whereas the same is highest for the K-Means. The F-Score is also highest in this case for K-Means.
\section{Conclusion} \label{sec_con}
In this paper, we have proposed a deterministic algorithm based on dynamical system theory for identifying the cluster centers in large data sets. Cluster centers are recovered as the asymptotically stable fixed points of the multi-agent gradient dynamical system, when the system is initialized at the initial conditions corresponding to the locations of data points. Lyapunov based convergence proof is provided for the clustering algorithm. The proposed approach is computationally tractable for larger data sets, as it does not involve eigen-decomposition type computation of a large matrix. At the same time, the algorithm does not suffer from possible variations in outcomes over multiple runs similar to heuristic algorithms. The algorithm automatically determines the number of cluster from the convergence behavior. Simulation results are presented over several data sets, and comparisons are made between existing techniques, and the proposed algorithm. 
\bibliography{ref_svm_1}
\bibliographystyle{plain}
\end{document}